\documentclass{article}
\usepackage{times}
\usepackage{xcolor}
\usepackage{soul}
\usepackage[utf8]{inputenc}
\usepackage[small]{caption}

\usepackage{booktabs}

\usepackage{helvet}
\usepackage{courier}
\usepackage{url}
\usepackage{graphicx}
\usepackage{amssymb}
\usepackage{amsmath}
\usepackage[ruled,vlined,linesnumbered]{algorithm2e}
\usepackage{bm}
\usepackage{subfigure}
\usepackage{graphicx}
\usepackage{epstopdf}

\usepackage{amsthm}
\usepackage{array}
\usepackage{bbm}

\usepackage{enumerate}

\usepackage{bookmark}
\usepackage{hyperref}

\newtheorem{theorem}{Theorem}
\newtheorem{lemma}[theorem]{Lemma}
\newtheorem{corollary}[theorem]{Corollary}
\newtheorem{fact}[theorem]{Fact}
\newtheorem{assumption}{Assumption}

\title{Finite Sample Analysis of LSTD with Random Projections and Eligibility Traces}

\author{
Haifang Li$^1$,
Yingce Xia$^2$ \ and
Wensheng Zhang$^1$
\\
$^1$ Institute of Automation, Chinese Academy of Sciences, Beijing, China \\
$^2$ University of Science and Technology of China, Hefei, Anhui, China\\
haifang.li@ia.ac.cn,
yingce.xia@gmail.com,
wensheng.zhang@ia.ac.cn
}

\begin{document}
\date{}
\maketitle

\begin{abstract}
Policy evaluation with linear function approximation is an important problem in reinforcement learning. When facing high-dimensional feature spaces, such a problem becomes extremely hard considering the computation efficiency and quality of approximations. We propose a new algorithm, LSTD($\lambda$)-RP, which leverages random projection techniques and takes eligibility traces into consideration to tackle the above two challenges. We carry out theoretical analysis of LSTD($\lambda$)-RP, and provide meaningful upper bounds of the estimation error, approximation error and total generalization error. These results demonstrate that LSTD($\lambda$)-RP can benefit from random projection and eligibility traces strategies, and LSTD($\lambda$)-RP can achieve better performances than prior LSTD-RP and LSTD($\lambda$) algorithms.
\end{abstract}

\section{Introduction}

Policy evaluation, commonly referred to as value function approximation, is an important and central part in many reinforcement learning (RL) algorithms
~\cite{sutton1998reinforcement},
whose task is to estimate value functions for a fixed policy in a discounted Markov Decision Process (MDP) environment. The value function of each state specifies the accumulated reward an agent would receive in the future by following the fixed policy from that state. Value functions have been widely investigated in 
RL applications, and it can provide insightful and important information for the agent to obtain an optimal policy, such as important board configurations
in Go \cite{silver2007reinforcement}, failure probabilities of large telecommunication networks \cite{frank2008reinforcement}, taxi-out times at large airports \cite{balakrishna2010accuracy} 
and so on.

Despite the value functions can be approximated by different ways, the simplest form, linear approximations, are still widely adopted and studied 
due to their good generalization abilities, relatively efficient computation and solid theoretical guarantees\cite{sutton1998reinforcement,dann2014policy,geist2014off,liang2016state}. Temporal Difference (TD) learning is a common approach to this policy evaluation with linear function approximation problem\cite{sutton1998reinforcement}. These typical TD algorithms can be divided into two categories: gradient based methods (e.g., GTD($\lambda$)~\cite{sutton2009convergent}) and least-square (LS) based methods (e.g., LSTD($\lambda$)\cite{boyan2002technical}). A good survey on these algorithms can be found in \cite{maei2011gradient,dann2012algorithms,geist2013algorithmic,dann2014policy,geist2014off}.


As the development of information technologies, high-dimensional data is widely seen in RL applications~\cite{sutton1996generalization,tedrake2004stochastic,riedmiller2007experiences}
, which brings serious challenges to design scalable and computationally efficient algorithms for the linear value function approximation problem. To address this practical issue, several approaches have been developed for efficient and effective value function approximation.
\cite{kolter2009regularization} and \cite{farahmand2011regularization} adopted $l_1$ or $l_2$ regularization techniques to control the complexity of the large function space and designed several $l_1$ and $l_2$ regularized RL algorithms. \cite{gehring2016incremental} studied this problem by using low-rank approximation via an incremental singular value function decomposition and proposed t-LSTD($\lambda$). \cite{pan2017accelerated} derived ATD($\lambda$) by combining the low-rank approximation and quasi-Newton gradient descent ideas. 

Recently, \cite{ghavamzadeh2010lstd} and \cite{pan2017effective} investigated sketching (projecting) methods to reduce the dimensionality in order to make it feasible to employ Least-Squares Temporal Difference (briefly, LSTD) algorithms. Specifically, \cite{ghavamzadeh2010lstd} proposed an approach named LSTD-RP, which is based on random projections. They showed that LSTD-RP can benefit from the random projection strategy. The eligibility traces have already been proven to be important parameters to control the quality of approximation during the policy evaluation process, but \cite{ghavamzadeh2010lstd} did not take them into consideration. \cite{pan2017effective} empirically investigated the effective use of sketching methods including random projections, count sketch, combined sketch and hadamard sketch for value function approximation, but they did not provide any conclusion on finite sample analysis. However, finite sample analysis is important for these algorithms since it clearly demonstrates the effects of the number of samples, dimensionality of the function space and the other related parameters.

In this paper, we focus on exploring the utility of random projections and eligibility traces on LSTD algorithms to tackle the computation efficiency and quality of approximations challenges in the high-dimensional feature spaces setting.
We also provide finite sample analysis to evaluate its performance.
To the best of our knowledge, this is the first work that performs formal finite sample analysis of LSTD with random projections and eligibility traces. Our contributions can be summarized from the following two aspects:

\emph{Algorithm}: By introducing random projections and eligibility traces, we propose a refined algorithm named 
\emph{LSTD with Random Projections and Eligibility Traces}
(denoted as LSTD($\lambda$)-RP for short), where $\lambda$ is the trace parameter of $\lambda$-return when considering eligibility traces.
LSTD($\lambda$)-RP algorithm consists of two steps: first, generate a low-dimensional linear feature space through random projections from the original high-dimensional feature space; then, apply LSTD($\lambda$) to this generated low-dimensional feature space.


\emph{Theoretical Analysis}: We perform theoretical analysis to evaluate the performance of LSTD($\lambda$)-RP 
and provide its finite sample performance bounds, 
including the estimation error bound, approximation error bound and total error bound. The analysis of the prior works LSTD-RP and LSTD($\lambda$) cannot directly apply to our setting, since (i) The analysis of LSTD-RP is based on a model of regression with Markov design,
but it does not hold when we incorporate eligibility traces; (ii) Due to utilizing random projections, the analysis of LSTD($\lambda$) cannot be directly used, especially the approximation error analysis.
To tackle these challenges, we first prove the linear independence property can be preserved by random projections, which is important for our analysis. Second, we decompose the total error into two parts: estimation error and approximation error. Then we make analysis on any fixed random projection space, and bridge these error bounds between the fixed random projection space and any arbitrary random projection space by leveraging
the norm and inner-product preservation properties of random projections,
the relationship between the smallest eigenvalues of the Gram matrices in the original and randomly projected spaces 
and the Chernoff-Hoeffding inequality for stationary $\beta$-mixing sequence. What's more, our theoretical results show that

\begin{enumerate}[1)]
\item Compared to LSTD-RP, the parameter $\lambda$ of eligibility traces illustrates a trade-off between the estimation error and approximation error for LSTD($\lambda$)-RP. We could tune $\lambda$ to select an optimal $\lambda^*$ which could balance these two errors and obtain the smallest total error bound. Furthermore, for fixed sample $n$, optimal dimension of randomly projected space $d^*$ in LSTD($\lambda$)-RP is much smaller than that of LSTD-RP.
\item Compared to LSTD($\lambda$), in addition to the computational gains which are the result of random projections, the estimation error of LSTD($\lambda$)-RP is much smaller at the price of a controlled increase of the approximation error. LSTD($\lambda$)-RP may have a better performance than LSTD($\lambda$), whenever the additional term in the approximation error is smaller than the gain achieved in the estimation error.
\end{enumerate}
These results demonstrate that LSTD($\lambda$)-RP can benefit from eligibility traces and random projections strategies in computation efficiency and approximation quality, and can be superior to LSTD-RP and LSTD($\lambda$) algorithms.


\section{Background}
\label{sec:background}

In this section, first we introduce some notations and preliminaries.
Then we make a brief review of LSTD($\lambda$) and LSTD-RP algorithms.

Now we introduce some notations for the following paper. Let $\lvert \cdot \rvert$ denote the size of a set and $\lVert \cdot \rVert_2$ denote the $L_2$ norm for vectors. Let $\mathcal{X}$ be a measurable space. Denote $\mathcal{S}(\mathcal{X})$ the set of probability measure over $\mathcal{X}$, and denote the set of measurable functions defined on $\mathcal{X}$ and bounded by $L\in \mathbb{R}^+$ as $\mathcal{B}(\mathcal{X},L)$. For a measure $\mu \in \mathcal{S}(\mathcal{X})$, the $\mu$-weighted $L_2$ norm of a measurable function $f$ is defined as $\lVert f \rVert_\mu=\sqrt{\sum_{x\in \mathcal{X}}f(x)^2\mu(x)}$. The operator norm for matrix $W$ is defined as $\lVert W \rVert_\mu=\sup_{w\neq 0}\frac{\lVert Ww\rVert_\mu}{\lVert w\rVert_\mu}$.

\subsection{Value Functions}

Reinforcement learning (RL) 
is an approach to find optimal policies in sequential decision-making problems, in which the RL agent interacts with a stochastic environment formalized by a discounted \emph{Markov Decision Process} (MDP) \cite{puterman2014markov}. An MDP is described as a tuple $\mathcal{M}=(\mathcal{X}, \mathcal{A}, \mathcal{P}_{xx'}^a,\mathcal{R}, \gamma)$, where state space $\mathcal{X}$ is finite \footnote{For simplicity, we assume the state space is finite. However, the results in this paper can be generalized into other more general state spaces.}, action space $\mathcal{A}$ is finite, $\mathcal{P}_{xx'}^a$ is the transition probability from state $x$ to the next state $x'$ when taking action $a$, $\mathcal{R}: \mathcal{X} \times \mathcal{A} \rightarrow \mathbb{R}$ is the reward function, which is uniformly bound by $R_{\max}$, and $\gamma \in (0,1)$ is the discount factor. A deterministic policy\footnote{Without loss of generality, here we only consider the deterministic policy. The extension to stochastic policy setting is straight-forward.} $\pi: \mathcal{X} \rightarrow \mathcal{A}$ is a mapping from state space to action space, which is an action selection policy. Given the policy $\pi$, the MDP $\mathcal{M}$ can be reduced to a Markov chain $\mathcal{M}^\pi=(\mathcal{X},P^\pi,r^\pi, \gamma)$, with transition probability $P^\pi(\cdot|x)=P(\cdot|x,\pi(x))$ and reward $r^\pi(x)= \mathcal{R}(x,\pi(x))$.

In this paper, we are interested in policy evaluation, which can be used to find optimal policies or select actions. It involves computing the state-value function of a given policy which assigns to each state a measure of long-term performance following the given policy. Mathematically, given a policy $\pi$, for any state $x \in \mathcal{X}$, the value function of state $x$ is defined as follows:
\begin{small}
\begin{center}
$V^{\pi}(x)=\mathbb{E}_\pi[\sum\nolimits_{t=0}^\infty \gamma^tr(X_t)|X_0=x],$
\end{center}
\end{small}
where $\mathbb{E}_\pi$ denotes the expectation over random samples which are generated by following policy $\pi$.
Let $V^{\pi}$ denote a vector constructed by stacking the values of $V^{\pi}(1),...,V^{\pi}(\lvert \mathcal{X} \rvert)$ on top of each other.
Then, we can see that $V^{\pi}$ is the unique fixed point of the Bellman operator $T^\pi$:
\begin{small}
\begin{equation}
\label{eqn:bellman_equation}
V^\pi=T^\pi V^\pi \overset{\Delta} {=} R^\pi+\gamma P^\pi V^\pi,
\end{equation}
\end{small}
where $R^\pi$ is the expected reward vector under policy $\pi$.
Equation (\ref{eqn:bellman_equation}) is called Bellman Equation, which is the basis of temporal difference learning approaches.
In the reminder of this paper, we omit the policy superscripts for ease of reference in unambiguous cases, since we are interested in on-policy learning in this work.

When the size of state space $\lvert \mathcal{X} \rvert$ is very large or even infinite, one may consider to approximate the state-value function by a linear function approximation, which is widely used in RL ~\cite{sutton1998reinforcement,dann2014policy}. We define a linear function space $\mathcal{F}$, which is spanned by the basis functions $\phi_i: \mathcal{X} \rightarrow \mathbb{R}, i \in [D] (D \ll \lvert \mathcal{X} \rvert)$\footnote{$[D]=\{1,...,D\}.$}, i.e., $\mathcal{F}= \{f_\alpha| f_\alpha(\cdot)=\phi(\cdot)^T\alpha, \alpha \in \mathbb{R}^D\}$, where $\phi(\cdot)=(\phi_1(\cdot),...,\phi_D(\cdot))^T$ is the feature vector.
We assume $\phi_i \in \mathcal{B}(\mathcal{X},L), i \in [D]$ for some finite positive constant $L$.
For any function $f_\alpha\in\mathcal{F}$, let $m(f_\alpha):= \lVert \alpha\rVert_2\sup_{x\in\mathcal{X}}\lVert \phi(x)\rVert_2$. Furthermore, we generate a $d$-dimensional ($d < D$) random space $\mathcal{G}$ from $\mathcal{F}$ through random projections $H$, where $H \in \mathbb{R}^{d\times D}$ be a random matrix whose each element is drawn independently and identically distributed (i.i.d.) from Gaussion distribution $\mathcal{N}(0,1/d)$\footnote{It is also can be some sub-Gaussian distributions. Without loss of generality, here we only consider Gaussian distribution for simplicity.}. For any $j \in [d]$, denote the randomly projected feature vector $\psi(\cdot)=(\psi_1(\cdot),...,\psi_i(\cdot))^T$, where $\psi(\cdot)=H\phi(\cdot)$. Thus, $\mathcal{G}=\{g_\beta\big\vert g_\beta(\cdot)=\psi(\cdot)^T\beta, \beta\in\mathbb{R}^d\}$. Define $\Phi=(\phi(x))_{x\in\mathcal{X}}=(\phi_1,\dots,\phi_D)$ of dimension $|\mathcal{X}|\times D$ and $\Psi=(\psi(x))_{x\in\mathcal{X}}=(\psi_1,\dots,\psi_D)$ of dimension $|\mathcal{X}|\times d$ to be the original and randomly projected feature matrix respectively.

\subsection{LSTD(\texorpdfstring{$\lambda$}))}

Least-Squares Temporal Difference (LSTD) is a traditional and important approach for policy evaluation in RL, which was first introduced by \cite{bradtke1996linear}, and later was extended to include the eligibility traces by \cite{boyan1999least,boyan2002technical} referred to as LSTD($\lambda$).

The essence of LSTD($\lambda$) is to estimate the fixed point of the projected multi-step Bellman equation, that is,
\begin{small}
\begin{equation}
\begin{aligned}
&V=\Pi_{\mathcal{F}} T^\lambda V, \\
\text{\normalsize{where}} \ \ & V= \Phi\theta, \
\text{\normalsize{and}}\ \ \Pi_{\mathcal{F}}=\Phi(\Phi^TD_\mu\Phi)^{-1}\Phi^TD_\mu,
\end{aligned}
\end{equation}
\end{small}
where $\mu$ is the steady-state probabilities of the Markov chain $\mathcal{M}^\pi$ induced by policy $\pi$, $D_\mu$ denotes the diagonal matrix with diagonal elements being $\mu$, $\Pi_{\mathcal{F}}$ is the orthogonal projection operator into the linear function space $\mathcal{F}$, and $T^{\lambda}$ is a multi-step Bellman operator parameterized by $\lambda \in [0,1]$, and it is defined as follows:
\begin{small}
\begin{center}
\label{eqn:T_lambda_expression}
$\ T^\lambda =(1-\lambda)\sum\nolimits_{i=0}^{\infty}\lambda^iT^{i+1}.$
\end{center}
\end{small}
When $\lambda =0$, we have $T^\lambda = T$, and it becomes LSTD.

Given one sampled trajectory $\{X_t\}_{t=1}^n$ generated by the Markov chain $\mathcal{M}^\pi$ under policy $\pi$,
the LSTD($\lambda$) algorithm returns $\hat{V}_{\text{LSTD}(\lambda)}=\Phi\tilde{\theta},$ with $\tilde{\theta}=\tilde{A}^{-1}\tilde{b}$, where
\begin{small}
\begin{equation}
\label{eqn:LSTD}
\begin{aligned}
&\tilde{A}=\frac{1}{n-1}\sum\nolimits_{i=1}^{n-1}\tilde{z}_i(\phi(X_i)-\gamma\phi(X_{i+1}))^T, \\ \text{\normalsize{and}} \quad  &\tilde{b}= \frac{1}{n-1}\sum\nolimits_{i=1}^{n-1}\tilde{z}_ir(X_i),
\end{aligned}
\end{equation}
\end{small}
where $\tilde{z}_i =\sum_{k=1}^i (\lambda\gamma)^{i-k}\phi(X_k)$ is called the eligibility trace, and $\lambda\in [0,1]$ is the trace parameter for the $\lambda$-return.

\subsection{LSTD-RP}

Compared to gradient based temporal difference (TD) learning algorithms, LSTD($\lambda$) has data sample efficiency and parameter insensitivity advantages, but it is less computationally efficient. LSTD($\lambda$) requires $O(D^3)$ computation per time step or still requires $O(D^2)$ by using the Sherman-Morrison formula to make incremental update. This expensive computation cost makes LSTD($\lambda$) impractical for the high-dimensional feature spaces scenarios in RL.
Recently, Least-Squares TD with Random Projections algorithm (briefly denoted as LSTD-RP)  was proposed to deal with the high-dimensional data setting \cite{ghavamzadeh2010lstd}.

The basic idea of LSTD-RP is to learn the value function of a given policy from a low-dimensional linear space $\mathcal{G}$ which is generated through random projections from a high-dimensional space $\mathcal{F}$.
Their theoretical results show that the total computation complexity of LSTD-RP is $O(d^3+ndD)$, which is dramatically less than the computation cost in the high dimensional space $\mathcal{F}$ (i.e., $O(D^3+nD^2)$). In addition to these practical computational gains, \cite{ghavamzadeh2010lstd} demonstrate that LSTD-RP can provide an efficient and effective approximation for value functions, since LSTD-RP reduces the estimation error at the price of the increase in the approximation error which is controlled.

However, LSTD-RP does not take the eligibility traces into consideration, which are important parameters in RL. First, the use of these traces can significantly speed up
learning by controlling the trade off between bias and variance  \cite{att2000bias,sutton2014new}. Second, the parameter $\lambda$ of these traces is also known to control the quality of approximation \cite{tsitsiklis1997analysis}. In the remainder of this paper, we present a generalization of LSTD-RP to deal with the $\lambda > 0$ scenario (i.e., LSTD($\lambda$)-RP (see Section \ref{sec:clstd_eligibility_trace})). What's more, we also give its theoretical guarantee in Section \ref{sec:err_analysis}.

\section{Algorithm}
\label{sec:clstd_eligibility_trace}

In this section, we first consider the Bellman equation with random projections (see Equation (\ref{eqn:projec_bellman_eqn})), and explore the existence and uniqueness properties of its solution, which is the goal of our newly proposed algorithm to estimate. Then we present the 
\emph{LSTD with Random Projections and Eligibility Traces} algorithm (briefly denoted as LSTD($\lambda$)-RP) as shown in Algorithm \ref{alg:LSTD-RP_Eligibility_Trace}, and discuss its computational cost.

\subsection{Bellman Equation with Random Projections}
To begin with, we make the following assumption throughout the paper as \cite{tsitsiklis1997analysis,tagorti2015rate}.
\begin{assumption}
\label{asmp: feature_linearly_independent}
The feature matrix $\Phi$ has full column rank; that is,
the original high-dimensional feature vectors $(\phi_j)_{j\in\{1,...,D\}}$ are linearly independent.
\end{assumption}

From the following lemma, we can get that the linear independence property can be preserved by random projections.
 Due to the space restrictions, we leave its detailed proof into Appendix B.
\begin{lemma}
\label{pro:projected_feature_indep}
Let Assumption \ref{asmp: feature_linearly_independent} hold. Then the randomly projected low-dimensional feature vectors $(\psi_j)_{j\in\{1,...,d\}}$
are linearly independent a.e.. Accordingly, $\Psi^TD_\mu\Psi$ is invertible a.e..\footnote {Notice that here the randomness is w.r.t. the random projection rather than the random sample. In the following paper, without loss of generality, we can assume $(\psi_j)_{j\in\{1,...,d\}}$ are linearly independent and $\Psi^TD_\mu\Psi$ is invertible.}
\end{lemma}

Let $\Pi_{\mathcal{G}}$ denote the orthogonal projection onto the randomly projected low-dimensional feature space $\mathcal{G}$ with respect to the $\mu$-weighted $L_2$ norm. According to Lemma \ref{pro:projected_feature_indep}, we obtain the projection $\Pi_{\mathcal{G}}$ has the following closed form
\begin{small}
\begin{center}
\label{eqn:orthogonal_proj}
$\Pi_{\mathcal{G}} = \Psi(\Psi^TD_\mu\Psi)^{-1}\Psi^TD_\mu.$
\end{center}
\end{small}
Then the projected multi-step \emph{Bellman equation with random projections} becomes
\begin{small}
\begin{equation}
\begin{aligned}
\label{eqn:projec_bellman_eqn}
&V=\Pi_{\mathcal{G}}T^{\lambda}V, \ \lambda\in [0,1],\\
\text{\normalsize{where}} \ \  &T^\lambda =(1-\lambda)\sum\nolimits_{i=0}^{\infty}\lambda^iT^{i+1}.
\end{aligned}
\end{equation}
\end{small}
Note that when $\lambda =0$, we have $T^\lambda = T$.

According to the Banach fixed point theorem, in order to guarantee the existence and uniqueness of the fixed point of Bellman equation with random projections (see Equation (\ref{eqn:projec_bellman_eqn})), we only need to demonstrate the contraction property of operator $\Pi_{\mathcal{G}}T^{\lambda}$. By simple derivations, we can demonstrate that the contraction property of $\Pi_{\mathcal{G}}T^{\lambda}$ holds as shown in the following Lemma \ref{pro:contraction}, and we leave its detailed proof into Appendix C.
\begin{lemma}
\label{pro:contraction}
Let Assumption \ref{asmp: feature_linearly_independent} hold. Then the projection operator $\Pi_{\mathcal{G}}$ is non-expansive w.r.t. $\mu$-weighted quadratic norm, and the operator $\Pi_{\mathcal{G}}T^{\lambda}$ is a ($\frac{\gamma(1-\lambda)}{1-\gamma\lambda}$-)contraction.
\end{lemma}

Denote the unique solution of the Bellman equation with random projections (see Equation (\ref{eqn:projec_bellman_eqn})) as $V_{\text{LSTD}(\lambda)\text{-RP}}$. In this work, we focus exclusively on the linear function approximation problem. Therefore, there exists $\theta \in \mathbb{R}^d$ such that
\begin{small}
\begin{equation}
V_{\text{LSTD}(\lambda)\text{-RP}}=\Psi\theta=\Pi_{\mathcal{G}}T^{\lambda}\Psi\theta.
\end{equation}
\end{small}
Just as the derivations of LSTD($\lambda$) algorithm \cite{tsitsiklis1997analysis,sutton1998reinforcement,boyan2002technical}, we can obtain that $\theta$ is a solution of the linear equation
\begin{small}
\begin{equation}
\label{eqn:bellman_solution_eqn}
\begin{aligned}
&A\theta=b,\\
\text{\normalsize{where}} \ \ & A=\Psi^TD_\mu(I-\gamma P)(I-\lambda\gamma P)^{-1}\Psi,\\
&\text{\normalsize{and}} \ \ b=\Psi^TD_\mu(I-\gamma\lambda P)^{-1}r.
\end{aligned}
\end{equation}
\end{small}
Furthermore, by Lemma \ref{pro:projected_feature_indep}, we can prove that $A$ is invertible. Thus, $V_{\text{LSTD}(\lambda)\text{-RP}}=\Psi A^{-1}b$ is well defined.

\subsection{LSTD(\texorpdfstring{$\lambda$}))-RP Algorithm}

Now we present our proposed algorithm LSTD($\lambda$)-RP in Algorithm \ref{alg:LSTD-RP_Eligibility_Trace}, which aims to estimate the solution of Bellman equation with random projections (see Equation (\ref{eqn:bellman_solution_eqn})) by using one sample trajectory $\{X_t\}_{t=1}^n$ generated by the Markov chain $\mathcal{M}^\pi$. Then we discuss its computational advantage compared to LSTD($\lambda$) and LSTD-RP.

\begin{algorithm}[!htbp]
\caption{LSTD($\lambda$)-RP Algorithm}
\label{alg:LSTD-RP_Eligibility_Trace}
\emph{Input}: The original high-dimensional feature vector $\phi: \mathcal{X} \rightarrow \mathbb{R}^D$; discount factor $\gamma \in [0,1)$; eligibility trace parameter $\lambda \in [0,1]$; the sample trajectory $\{X_t, r_t\}_{t=1}^n$, where $X_t$ and $r_t$ are the observed state and reward received at time $t$ respectively\;
\emph{Output}: $\hat{\theta}:=\hat{A}^{-1}\hat{b}$ or $\hat{\theta}:=\hat{A}^\dagger\hat{b}$, where $\hat{A}^\dagger$ denote the Moore-Penrose pseudo-inverse of matrix $\hat{A}$\;
Initialize: $\hat{A} \leftarrow 0, \hat{b}\leftarrow 0, z\leftarrow 0, t \leftarrow 0$\;
Generate random projection matrix $H\in \mathbb{R}^{d\times D}$ whose elements are drawn i.i.d. from $\mathcal{N}(0,1/d)$\;
{\For{$t=0,1,\dots,n$}{
$t\leftarrow t+1$\;
The randomly projected low-dimensional feature vector $\psi(X_t)=H\phi(X_t)$\;
$z\leftarrow \lambda\gamma z +\psi(X_t)$\;
$\Delta\hat{A}\leftarrow z(\psi(X_t)-\psi(X_{t+1}))^T$\;
$\Delta\hat{b}\leftarrow zr_t$\;
$\hat{A}\leftarrow \hat{A}+\frac{1}{t}[\Delta\hat{A}-\hat{A}]$\;
$\hat{b}\leftarrow \hat{b}+\frac{1}{t}[\Delta\hat{b}-\hat{b}]$\;
}}
\end{algorithm}
LSTD($\lambda$)-RP algorithm is a generalization of LSTD-RP. It uses eligibility traces to handle the $\lambda > 0$ case. Line 8 updates the eligibility traces $z$, and lines 9-12 incrementally update $A$ and $b$ as described in Equation (\ref{eqn:LSTD_eligi}), which have some differences from that in LSTD-RP algorithm due to eligibility traces. If the parameter $\lambda$ is set to zero, then the LSTD($\lambda$)-RP algorithm becomes the original LSTD-RP algorithm. What's more, if the random projection matrix $H$ is identity matrix, then LSTD($\lambda$)-RP becomes LSTD($\lambda$).

From Algorithm \ref{alg:LSTD-RP_Eligibility_Trace}, we obtain that the LSTD($\lambda$)-RP algorithm returns
\begin{small}
\begin{equation}
\hat{V}_{\text{LSTD}(\lambda)\text{-RP}}=\Psi\hat{\theta},
\end{equation}
\end{small}
with $\hat{\theta}=\hat{A}^{-1}\hat{b},$\footnote{We will see that $\hat{A}^{-1}$ exists with high probability for a sufficiently large sample size $n$ in Theorem \ref{thm:uniqueness_LSTD_sample_based_solution}.} where
\begin{small}
\begin{equation}
\label{eqn:LSTD_eligi}
\begin{aligned}
&\hat{A}=\frac{1}{n-1}\sum\nolimits_{i=1}^{n-1}z_i(\psi(X_i)-\gamma\psi(X_{i+1}))^T, \\
& \hat{b}=\frac{1}{n-1}\sum\nolimits_{i=1}^{n-1}z_ir(X_i), \
\text{\normalsize{and}}\ z_i=\sum\nolimits_{k=1}^{i}(\lambda\gamma)^{i-k}\psi(X_k).
\end{aligned}
\end{equation}
\end{small}
Here $z_i$  is referred to as \emph{randomly projected eligibility trace}.

The difference between LSTD($\lambda$)-RP algorithm and the prior LSTD-RP algorithm lies in the fact that LSTD($\lambda$)-RP incorporates the eligibility traces. From Algorithm \ref{alg:LSTD-RP_Eligibility_Trace}, we know that the computational cost of eligibility traces is $O(nd)$. Based on the analysis of the computational complexity of LSTD-RP algorithm, we obtain that the total computational complexity of LSTD($\lambda$)-RP is $O(d^3+ndD)(D \gg d)$. This reveals that the computation cost of LSTD($\lambda$)-RP algorithm is much less than that of LSTD($\lambda$) algorithm, which is $O(D^3+nD^2)$ \cite{ghavamzadeh2010lstd}.

To evaluate the performance of LSTD($\lambda$)-RP algorithm, we consider the gap between the value function learned by LSTD($\lambda$)-RP algorithm $\hat{V}_{\text{LSTD}(\lambda)\text{-RP}}$ and the true value function $V$, i.e., $\lVert \hat{V}_{\text{LSTD}(\lambda)\text{-RP}}-V\rVert_\mu$. We refer to this gap as the \emph{total error} of the LSTD($\lambda$)-RP algorithm. According to the triangle inequality, we can decompose the total error
into two parts: \emph{estimation error} $\lVert \hat{V}_{\text{LSTD}(\lambda)\text{-RP}}-V_{\text{LSTD}(\lambda)\text{-RP}}\rVert_\mu$ and \emph{approximation error} $\lVert V_{\text{LSTD}(\lambda)\text{-RP}}-V\rVert_\mu$. We will illustrate how to derive meaningful upper bounds for these three errors of LSTD($\lambda$)-RP in the following section.

\section{Theoretical Analysis}
\label{sec:err_analysis}

In this section, we conduct theoretical analysis for LSTD($\lambda$)-RP. First, we examine the sample size needed to ensure the uniqueness of the sample-based LSTD($\lambda$)-RP solution, that is, we explore sufficient conditions to guarantee the invertibility of $\hat{A}$ with high probability, which can be used in the analysis of estimation error bound. Second, we make finite sample analysis of LSTD($\lambda$)-RP including discussing how to derive meaningful upper bounds for the estimation error $\lVert \hat{V}_{\text{LSTD}(\lambda)\text{-RP}}-V_{\text{LSTD}(\lambda)\text{-RP}}\rVert_\mu$, the approximation error $\lVert V_{\text{LSTD}(\lambda)\text{-RP}}-V\rVert_\mu$ and the total error $\lVert\hat{V}_{\text{LSTD}(\lambda)\text{-RP}}-V\rVert_\mu$.

To perform such finite sample analysis, we also need to make a common assumption on the Markov chain process $(X_t)_{t\geq 1}$ that has some $\beta$-mixing properties as shown in Assumption \ref{asmp:mixing_sequence} \cite{mohri2010stability,tagorti2015rate}. Under this assumption, we can make full use of the concentration inequality for $\beta$-mixing sequences during the process of finite sample analysis.

\begin{assumption}
\label{asmp:mixing_sequence}
$(X_t)_{t\geq 1}$ is a stationary exponential $\beta$-mixing sequence, that is, there exist some constant parameters $\beta_0>0$, $\beta_1>0,$ and $\kappa> 0$ such that $\beta(m)\leq \beta_0\exp(-\beta_1m^{\kappa})$.
\end{assumption}

\subsection{Uniqueness of the Sample-Based Solution }

In this subsection, we explore how sufficiently large the number of observations $n$ needed to guarantee the invertibility of $\hat{A}$  with high probability as shown in Theorem \ref{thm:uniqueness_LSTD_sample_based_solution}, which indicates the uniqueness of sample-based LSTD($\lambda$)-RP solution. Due to the space limitations,
we leave the detailed proof into Appendix D.

\begin{theorem}
\label{thm:uniqueness_LSTD_sample_based_solution}
Let Assumptions \ref{asmp: feature_linearly_independent} and \ref{asmp:mixing_sequence} hold, and $X_1 \sim \mu$.
For any $\delta \in (0,1), \gamma \in (0,1), $ and $\lambda\in[0,1],$
let $n_0(\delta)$ be the smallest integer such that
\begin{small}
\begin{equation}
\label{eqn:const_number_need}
\begin{aligned}
&\frac{2dL^2}{(1-\gamma)\nu_F\eta(d,D,\delta/2)}\bigg[\frac{2\xi(d,n,\delta/4)}{\sqrt{n-1}}\sqrt{(1+m_n^\lambda)I(n-1,\frac{\delta}{2})}\\
+&\frac{2\xi(d,n,\delta/4)}{n-1}m_n^\lambda+\frac{1}{(1-\lambda\gamma)(n-1)}\bigg]<1,
\end{aligned}
\end{equation}
\end{small}
where
\begin{small}
\begin{equation*}
\begin{aligned}
m_n^{\lambda}=&\left\{
\begin{array}{rcl}
\lceil \frac{\log(n-1)}{\log\frac{1}{\lambda\gamma}}\rceil & {\lambda \in (0,1]}\\
0 &{\lambda = 0}
\end{array} \right., \xi(n,d,\delta)=1+\sqrt{\frac{8}{d}\log\frac{n}{\delta}},\\
&\eta(d,D,\delta)=\big(1-\sqrt{d/D}-\sqrt{2\log(2/\delta)/D}\big)^2,\\
&I(n,\delta)= 32\Lambda(n,\delta)\max\{\Lambda(n,\delta)/\beta_1,1\}^{\frac{1}{\kappa}},\\
&\Lambda(n,\delta)=\log(8n^2/\delta)+\log(\max\{4e^2,n\beta_0\}),
\end{aligned}
\end{equation*}
\end{small}
and $\nu_F$ is the smallest eigenvalue of the Gram matrix $F=\Phi^TD_\mu\Phi$. Then when $D> d+2\sqrt{2d\log(4/\delta)}+2\log(4/\delta)$, 
with probability at least $1-\delta$ (the randomness w.r.t. the random sample and the random projection), we have, for all $n\geq n_0(\delta)$, $\hat{A}$ is invertible.
\end{theorem}


From Theorem \ref{thm:uniqueness_LSTD_sample_based_solution}, we can draw the following conclusions:
\begin{enumerate}[1)]
\item The number of observations needed to guarantee the uniqueness of the sample-based LSTD($\lambda$)-RP solution is of order $\tilde{O}(d^2)$,
    and it is much smaller than that of LSTD($\lambda$), which is of order $\tilde{O}(D^2) (D \gg d)$ (Theorem 1 in \cite{tagorti2015rate}).
\item In our analysis, setting $\lambda = 0$, we can see that our result has some differences from LSTD-RP (Lemma 3 in \cite{ghavamzadeh2010lstd}), since we consider the invertibility of the matrix $\hat{A}$, while they consider the empirical Gram matrix $\frac{1}{n}\Psi^T\Psi$.
\end{enumerate}

\noindent\emph{Remark 1:}
According to Assumption \ref{asmp: feature_linearly_independent}, we know that $\nu_F>0$. For all $\delta \in (0,1)$ and fixed $d$, $n_0(\delta)$ exists since the left hand side of Equation (\ref{eqn:const_number_need}) tends to 0 when $n$ tends to infinity.

\subsection{Estimation Error Bound}
\label{sec:est_err}

In this subsection, we upper bound the estimation error of LSTD($\lambda$)-RP as shown in Theorem \ref{thm:estimation_err}. For its proof, first, bound the estimation error on one fixed randomly projected space $\mathcal{G}$. Then,
by utilizing properties of random projections, the relationship between the smallest eigenvalues of the Gram matrices in $\mathcal{F}$ and $\mathcal{G}$ and the conditional expectation properties, bridge the error bounds
between the fixed space and any arbitrary random
projection space. Due to space limitations,
we leave its detailed proof into Appendix E.
\begin{theorem}
\label{thm:estimation_err}
Let Assumptions \ref{asmp: feature_linearly_independent} and \ref{asmp:mixing_sequence} hold, and let $X_1 \sim \mu$. For any $\delta \in (0,1)$, $\gamma \in (0,1), $ and $\lambda\in[0,1],$ when $D> d+2\sqrt{2d\log(4/\delta)}+2\log(4/\delta)$ and $d\geq 15\log(4n/\delta)$, with probability $1-\delta$ (the randomness w.r.t. the random sample and the random projection), for all $n\geq n_0(\delta),$ the estimation error \begin{small}$\lVert V_{\text{LSTD}(\lambda)\text{-RP}}-\hat{V}_{\text{LSTD}(\lambda)\text{-RP}}\rVert_{\mu}$\end{small} is upper bounded as follows:
\begin{small}
\begin{equation}
\label{eqn:estimation_error_bound}
\begin{aligned}
&\lVert V_{\text{LSTD}(\lambda)\text{-RP}}-\hat{V}_{\text{LSTD}(\lambda)\text{-RP}}\rVert_{\mu}\leq h(n,d,\delta)\\
+&\frac{4V_{\max}dL^2\xi(n,d,\delta/4)}{\sqrt{n-1}(1-\gamma)\nu_F\eta(d,D,\delta/2)}\sqrt{(m_n^\lambda+1)I(n-1,\delta/4)},
\end{aligned}
\end{equation}
\end{small}
with $h(n,d,\delta)=\tilde{O}(\frac{d}{n}\log\frac{1}{\delta})$, where $\nu_F (>0)$ is the smallest eigenvalue of the Gram matrix $\Phi^TD_u\Phi$, $V_{\max}=\frac{R_{\max}}{1-\gamma}$, $\xi(n,d,\delta)$, $\eta(d,D,\delta)$, $ m_n^\lambda$, $I(n,\delta),$ and $n_0(\delta)$ are defined as in Theorem \ref{thm:uniqueness_LSTD_sample_based_solution}.
\end{theorem}

From Theorem \ref{thm:estimation_err}, we have by setting $\lambda = 0$ in Equation (\ref{eqn:estimation_error_bound}), the estimation error bound of LSTD($\lambda$)-RP becomes of order $\tilde{O}(d/\sqrt{n})$, and it is consistent with that of LSTD-RP (Theorem 2 in \cite{ghavamzadeh2010lstd}).

\subsection{Approximation Error Bound}
\label{sec:appro_err}

Now we upper bound the approximation error of LSTD($\lambda$)-RP which is shown in Theorem \ref{thm:appr_err}. As to its proof, we first analyze the approximation error on any fixed random projected space $\mathcal{G}$. Then, we make a bridge of approximation error bound
between the fixed random projection space and any arbitrary random
projection space by leveraging the definition of projection and the inner-product preservation property of random projections and the Chernoff-Hoeffding
inequality for stationary $\beta$-mixing sequence. Due to space limitations, we leave detailed proof into Appendix F.

\begin{theorem}
\label{thm:appr_err}
Let Assumptions \ref{asmp: feature_linearly_independent} and \ref{asmp:mixing_sequence} hold. Let $X_1 \sim \mu$. For any $\delta \in (0,1), \gamma \in (0,1), $ and $\lambda\in[0,1],$ when $d\geq 15\log(8n/\delta)$, with probability at least $1-\delta$ (w.r.t. the random projection), the approximation error of LSTD($\lambda$)-RP algorithm $\lVert V-V_{\text{LSTD}(\lambda)\text{-RP}}\rVert_\mu$ can be upper bounded as below,
\begin{small}
\begin{equation}
\label{eqn:approx_error_bound}
\begin{aligned}
&\lVert V-V_{\text{LSTD}(\lambda)\text{-RP}}\rVert_\mu \leq \frac{1-\lambda\gamma}{1-\gamma}\big[\lVert V-\Pi_{\mathcal{F}}V\rVert_\mu\\
+&\sqrt{(8/d)\log(8n/\delta)}(1+\frac{2\sqrt{\Upsilon(n,\delta/2)}}{\sqrt{n}})m(\Pi_{\mathcal{F}}V)\big] ,
\end{aligned}
\end{equation}
\end{small}
where $\Upsilon(n,\delta)=(\log\frac{4+n\beta_0}{\delta})^{1+\frac{1}{\kappa}}\beta_1^{-\frac{1}{\kappa}}$.
\end{theorem}

From Theorem \ref{thm:appr_err}, we know that by setting $\lambda =0$, the right hand of Equation (\ref{eqn:approx_error_bound}) becomes
\begin{small}
$\frac{1}{1-\gamma}\big[\lVert V-\Pi_{\mathcal{F}}V\rVert_\mu
+O\big(\sqrt{(1/d)\log(n/\delta)}m(\Pi_{\mathcal{F}}V)\big)\big],$
\end{small}
while for LSTD-RP (Theorem 2 in \cite{ghavamzadeh2010lstd}) it is
\begin{small}
$\frac{4\sqrt{2}}{\sqrt{1-\gamma^2}}\big[\lVert V-\Pi_{\mathcal{F}}V\rVert_\mu
+O\big(\sqrt{(1/d)\log(n/\delta)}m(\Pi_{\mathcal{F}}V)\big)\big].$
\end{small}
Notice that they are just different from the coefficients. Furthermore, due to eligibility traces which can control the quality of approximation, we could tune $\lambda$ to make approximation error of LSTD($\lambda$)-RP smaller than that of LSTD-RP, since the coefficient in Equation (\ref{eqn:approx_error_bound}) is $\frac{1-\lambda\gamma}{1-\gamma}$, while it is $\frac{4\sqrt{2}}{\sqrt{1-\gamma^2}}$ in LSTD-RP.

\noindent \emph{Remark 2}: The coefficient $\frac{1-\lambda\gamma}{1-\gamma}$ in the approximation can be improved by $\frac{1-\lambda\gamma}{\sqrt{(1-\gamma)(1+\gamma-2\lambda\gamma)}}$ \cite{tsitsiklis1997analysis}.

\subsection{Total Error Bound}
\label{sec:generalization_bound}

Combining Theorem \ref{thm:estimation_err} and Theorem \ref{thm:appr_err}, and by leveraging the triangle inequality, we can obtain the total error bound for LSTD($\lambda$)-RP as shown in the following corollary.

\begin{corollary}
\label{thm:generalization_err}
Let Assumptions \ref{asmp: feature_linearly_independent} and \ref{asmp:mixing_sequence} hold. Let $X_1 \sim \mu$.
For any $\delta \in (0,1), \gamma \in (0,1), $ and $\lambda\in[0,1],$
when $D> d+2\sqrt{2d\log(8/\delta)}+2\log(8/\delta)$ and $d\geq 15\log(16n/\delta)$, with probability (the randomness w.r.t. the random sample and the random projection) at least $1-\delta$, for all $n\geq n_0(\delta)$, the total error \begin{small}$\lVert V- \hat{V}_{\text{LSTD}(\lambda)\text{-RP}}\rVert_\mu$\end{small} can be upper bounded by:
\begin{small}
\begin{equation}
\begin{aligned}
&\frac{4V_{\max}dL^2\xi(n,d,\delta/8)}{\sqrt{n-1}(1-\gamma)\nu_F\eta(d,D,\delta/4)}\sqrt{(m_n^\lambda+1)I(n-1,\delta/8)}\\
+&\frac{1-\lambda\gamma}{1-\gamma}\big[\lVert V-\Pi_{\mathcal{F}}V\rVert_\mu
+\sqrt{(8/d)\log(16n/\delta)}(1+\\
&(2/\sqrt{n})\sqrt{\Upsilon(n,\delta/4)})m(\Pi_{\mathcal{F}}V)\big]+h(n,d,\delta)
\end{aligned}
\end{equation}
\end{small}
with $h(n,d,\delta)=\tilde{O}(\frac{d}{n}\log\frac{1}{\delta})$, where $\nu_F (> 0)$ is the smallest eigenvalue of the Gram matrix $\Phi^TD_u\Phi$, $V_{\max}=\frac{R_{\max}}{1-\gamma}$, $\xi(n,d,\delta)$, $\eta(d,D,\delta)$, $m_n^\lambda$, $I(n,\delta),$ $n_0(\delta)$ are defined as in Theorem \ref{thm:uniqueness_LSTD_sample_based_solution} and $\Upsilon(n,\delta)$ is defined as in Theorem \ref{thm:appr_err}.
\end{corollary}

By setting $\lambda =0$, the total error bound of LSTD($\lambda$)-RP is consistent with that of LSTD-RP except for some differences in coefficients. These differences lie in the analysis of LSTD-RP based on a model of regression with Markov design.

Although our results consistent with LSTD-RP when setting $\lambda=0$ except for some coefficients, our results have some advantages over LSTD-RP and LSTD($\lambda$).
Now we have some discussions. From Theorem \ref{thm:estimation_err}, Theorem \ref{thm:appr_err} and Corollary \ref{thm:generalization_err}, we can obtain that
\begin{enumerate}[1)]
\item Compared to LSTD($\lambda$), the estimation error of LSTD($\lambda$)-RP is of order $\tilde{O}(d/\sqrt{n})$, which is much smaller than that of LSTD($\lambda$) (i.e., $\tilde{O}(D/\sqrt{n})$ (Theorem 1 in \cite{tagorti2015rate})), since random projections can make the complexity of the projected space $\mathcal{G}$ is smaller than that of the original high-dimensional space $\mathcal{F}$. Furthermore, the approximation error of LSTD($\lambda$)-RP increases by at most \begin{small}$O(\sqrt{(1/d)\log(n/\delta)}m(\Pi_{\mathcal{F}}V))$\end{small}, which decreases w.r.t $d$. This shows that in addition to the computational gains, the estimation error of LSTD($\lambda$)-RP is much smaller at the cost of a increase of the approximation error which can be fortunately controlled. Therefore, LSTD($\lambda$)-RP may have a better performance than LSTD($\lambda$), whenever the additional term in the approximation error is smaller than the gain achieved in the estimation error.
\item Compared to LSTD-RP, $\lambda$ illustrates a trade-off between the estimation error and approximation error for LSTD($\lambda$)-RP, since eligibility traces can control the trade off between the approximation bias and variance during the learning process.
    When $\lambda$ increases, the estimation error
    would increase, while the approximation error
    would decrease.
    Thus, we could select an optimal $\lambda^*$ to balance these two errors and obtain the smallest total error.
\item Compared to LSTD-RP, we can select an optimal \begin{small}$d^*_{\text{LSTD}(\lambda)\text{-RP}}= \tilde{O}(n\log n)^{\frac{1}{3}}$\end{small} to obtain the smallest total error, and make a balance between the estimation error and the approximation error of LSTD($\lambda$)-RP, which is much smaller than that of LSTD-RP (\begin{small}$d^*_{\text{LSTD-RP}}= \tilde{O}(n\log n)^{\frac{1}{2}}$\end{small}) due to the effect of eligibility traces.

\end{enumerate}

These conclusions demonstrate that random projections and eligibility traces can improve the approximation quality and computation efficiency.
Therefore, LSTD($\lambda$)-RP can provide an efficient and effective approximation for value functions and can be superior to LSTD-RP and LSTD($\lambda$).


\noindent\emph{Remark 3}: Some discussions about the role of factor $m(\Pi_{\mathcal{F}}V)$ in the error bounds can be found in \cite{maillard2009compressed} and \cite{ghavamzadeh2010lstd}.

\noindent\emph{Remark 4}: Our analysis can be simply generalized to the emphatic LSTD algorithm (ELSTD)\cite{yu2015convergence} with random projections and eligibility traces.

\section{Conclusion and Future Work}
\label{sec:conclusion}
In this paper, we propose a new algorithm LSTD($\lambda$)-RP, which leverages random projection techniques and takes eligibility traces into consideration to tackle the computation efficiency and quality of approximations challenges in the high-dimensional feature space scenario. We also make theoretical analysis for LSTD($\lambda$)-RP.

For the future work, there are still many important and interesting
directions: (1) the convergence analysis of the off-policy learning with random projections is worth studying; (2) the comparison of LSTD($\lambda$)-RP to $l_1$ and $l_2$ regularized approaches asks for further investigation. (3) the role of $m(\Pi_{\mathcal{F}}V)$ in the error bounds is in need of discussion.

\section*{Acknowledgments}
This work is partially supported by the National Key Research and Development Program of China (No. 2017YFC0803704 and No. 2016QY03D0501), the National Natural Science Foundation of China (Grant No. 61772525, Grant No. 61772524, Grant No. 61702517 and Grant No. 61402480) and the Beijing Natural Science Foundation (Grant No. 4182067 and Grant No. 4172063).

\section*{Appendix}
\appendix

\section{Preparations}



Now we present some useful facts (Fact \ref{lem:JJL}-\ref{lem:eigenvalue_g}), which are important for the following theoretical analysis processes. Specifically, Fact \ref{lem:JJL} and \ref{lem:random_linear_norm_preserve} show the norm and inner-product preservation properties of random projections respectively, and Fact \ref{lem:eigenvalue_g} states the relationship between the smallest eigenvalues of the Gram matrices in spaces $\mathcal{F}$ and $\mathcal{G}$.

\begin{fact}
\label{lem:JJL}\cite{frankl1988johnson}
Let $H\in \mathbb{R}^{d\times D}$ of i.i.d. elements drawn from $\mathcal{N}(0,\frac{1}{d})$. Then for any vector $u\in \mathbb{R}^D$, the random (w.r.t. the choice of the matrix $H$) variable $\lVert Hu \rVert_2^2$ concentrates around its expectation $\lVert u \rVert_2^2$. Mathematically, for any $\epsilon \in(0,1),$ we have
\begin{small}
\begin{displaymath}
\begin{aligned}
\mathbb{P}\{\big| \lVert Hu \rVert_2^2 - \lVert u\rVert_2^2 \big| \geq \epsilon\lVert u\rVert_2^2\}\leq 2\exp(-d(\epsilon^2/4-\epsilon^3/6)).
\end{aligned}
\end{displaymath}
\end{small}
\end{fact}
\begin{fact}{\cite{maillard2009compressed}}
\label{lem:random_linear_norm_preserve}
Let $(u_k)_{1\leq k\leq n}$ and $w$ be vectors of $\mathbb{R}^D$. Let $H\in \mathbb{R}^{d\times D}$ of i.i.d. elements drawn from $\mathcal{N}(0,\frac{1}{d})$. For any $\epsilon>0$, $\delta \in (0,1)$, for $d\geq \frac{1}{\frac{\epsilon^2}{4}-\frac{\epsilon^3}{6}}\log\frac{4n}{\delta}$, we have, with probability at least $1-\delta$, for all $k\leq n$,
\begin{small}
$\lvert Hu_k\cdot Hw-u_k\cdot w\rvert \leq \epsilon\lVert u_k\rVert_2\lVert w\rVert_2.$
\end{small}
\end{fact}

\begin{fact}{\cite{ghavamzadeh2010lstd}}
\label{lem:eigenvalue_g}
Let $\delta \in (0,1)$. $\mathcal{F}$ and $\mathcal{G}$ with dimensions $D$ and $d$ $(d<D)$ are defined in section 2 with $D>d+2\sqrt{2d\log(2/\delta)}+2\log(2/\delta)$. Let $F$ and $G$ be the Gram matrices for spaces $\mathcal{F}$ and $\mathcal{G}$ (i.e., $F=\Phi^TD_\mu\Phi$, $G=\Psi^TD_\mu\Psi$.), and $\nu_F$ and $\nu_G$ be their corresponding smallest eigenvalues. Then, with probability $1-\delta$ (w.r.t. the random projection), we have
\begin{small}
\begin{equation*}
\nu_G \geq (D/d)\nu_F\big(1-\sqrt{d/D}-\sqrt{(2\log(2/\delta))/D}\big)^2>0.
\end{equation*}
\end{small}
\end{fact}

The following fact gives a measure of the difference between the distribution of $m$ blocks where the blocks are independent in one case and dependent in the other case. The distribution within each block is assumed to be the same in both cases.
\begin{fact}{\rm(\cite{yu1994rates}, Corollary 2.7)}
\label{independent_block}
Let $l\geq 1$ and suppose that $h$ is a measurable function on a product probability space $(\Pi_{i=1}^l\Omega_i,\Pi_{i=1}^l\sigma_{k_i}^{s_i})$ with bound $M_h$, where $k_i\leq s_i \leq k_{i+1}$ for all $i$. Let $Q$ be a probability measure on the product space with marginal measures $Q_i$ on $(\Omega_i,\sigma_{k_i}^{s_i})$, and let $Q^{i+1}$ be the marginal measure of $Q$ on $(\Pi_{j=1}^{i+1}\Omega_j,\Pi_{j=1}^{i+1}\sigma_{k_j}^{s_j})$, $i=1,\dots,m-1.$ Let $\beta(Q)=\sup_{1\leq i\leq m-1}\beta(m_i),$ where $m_i=r_{i+1}-s_i$ and $P=\Pi_{i=1}^m Q_i$. Then,
\begin{displaymath}
\lvert\mathbb{E}_Q[h]-\mathbb{E}_P[h]\rvert \leq (m-1)M_h\beta(Q)
\end{displaymath}
\end{fact}

In addition, we present the key Fact \ref{lem:est_err_step_2_nonProj} for our analysis, which shows the concentration inequality holds for the infinitely-long-trace $\beta$-mixing process.

\begin{fact}{(\rm\cite{tagorti2015rate}, Lemma 2)}
\label{lem:est_err_step_2_nonProj}
Let Assumptions 1 and 2 hold and let $X_1 \sim \mu$. Define the $D\times k$ matrix $Q_i$, such that
\begin{equation}
Q_i=\sum_{l=1}^i (\lambda\gamma)^{i-l}\phi(X_l)(\tau(X_i,X_{i+1}))^T,
\end{equation}
where $\phi=(\phi_1,...,\phi_D)$, for all $i \in [1,D]$, $\phi_i\in \mathcal{B}(\mathcal{X},L)$, and $\tau_j\in\mathcal{B}(\mathcal{X}^2,L'), j\in[1,D]$. Then for any $\delta \in(0,1)$, with probability at least $1-\delta$, we have
\begin{small}
\begin{equation}
\begin{aligned}
&\lVert \frac{1}{n-1}\sum_{i=1}^{n-1}Q_i-\frac{1}{n-1}\sum_{i=1}^{n-1}\mathbb{E}[Q_i]\rVert_2 \\
\leq &\frac{2\sqrt{Dk}LL'}{(1-\lambda\gamma)\sqrt{n-1}}\sqrt{(1+m_n^\lambda)J(n-1,\delta)}
+2m_n^{\lambda}\frac{\sqrt{Dk}LL'}{(n-1)(1-\lambda\gamma)},\\
\end{aligned}
\end{equation}
\end{small}
where
\begin{small}
\begin{equation*}
\begin{aligned}
&m_n^\lambda=\lceil \frac{\log(n-1)}{\log\frac{1}{\lambda\gamma}}\rceil,\\
&J(n,\delta)= 32\Gamma(n,\delta)\max\{\frac{\Gamma(n,\delta)}{\beta_1},1\}^{\frac{1}{\kappa}},\\ &\Gamma(n,\delta)=\log(\frac{2}{\delta})+\log(\max\{4e^2,n\beta_0\}).
\end{aligned}
\end{equation*}
\end{small}
\end{fact}

\begin{fact}
\label{LSTD_non_RP_appr_err}
{\rm(\cite{tsitsiklis1997analysis}, Theorem 1)}
The LSTD($\lambda$) approximation error satisfies
\begin{small}
\begin{equation}
\lVert V-V_{LSTD(\lambda)}\rVert_\mu \leq \frac{1-\lambda\gamma}{1-\gamma}\lVert V-\Pi_{\mathcal{F}}V\rVert_\mu.
\end{equation}
\end{small}
\end{fact}

\section{Proof of Lemma 1}

\begin{proof}
Under Assumption 1, since $\Psi=\Phi H^T$,
to prove that $(\psi_j)_{j\in\{1,...,d\}}$ are linearly independent a.e., that is,
\begin{small}
\begin{equation*}
\mathbb{P}\{\Psi x_{d\times 1}=(\psi_1,...,\psi_d)x_{d\times 1}=0 \Rightarrow x_{d\times 1}=0\}=1.
\end{equation*}
\end{small}
we only need to show that
\begin{small}
\begin{equation*}
\mathbb{P}\{H^Tx_{d\times 1}=0 \Rightarrow x_{d\times 1}=0\}=1 \ \text{holds.}
\end{equation*}
\end{small}
Now decompose the random projection matrix $H\in\mathbb{R}^{d\times D}$ into two blocks as
\begin{small}
$H=
\begin{bmatrix}
H_1, H_2
\end{bmatrix},$
\end{small}
where $H_1 \in \mathbb{R}^{d\times d}$ and $H_2 \in \mathbb{R}^{d\times (D-d)}$. Since each element of $H_1$ are continuous variable, by mathematical induction, we can show that the determinant of matrix $H_1$ $|H_1|\neq 0$ a.e., which implies that
\begin{small}
$\mathbb{P}\{H_1^Tx_{d\times 1}=0 \Rightarrow x_{d\times 1}=0\}=1.$
\end{small}
Therefore, we have
\begin{small}
\begin{equation*}
\begin{aligned}
&\mathbb{P}\{H^Tx_{d\times 1}=0 \Rightarrow x_{d\times 1}=0\}\\
=&\mathbb{P}\{\begin{bmatrix}H_1^Tx_{d\times 1}\\ H_2^Tx_{d\times 1}\end{bmatrix}=\begin{bmatrix}0\\0 \end{bmatrix} \Rightarrow x_{d\times 1}=0\}\\
\geq &\mathbb{P}\{H_1^Tx_{d\times 1}=0 \Rightarrow x_{d\times 1}=0\}=1.
\end{aligned}
\end{equation*}
\end{small}
Furthermore, we have
\begin{small}
\begin{equation*}
\begin{aligned}
&\mathbb{P}\{\Psi^T D_\mu \Psi x_{d\times1}=0 \Rightarrow x_{d\times1}=0\}\\
\geq &\mathbb{P}\{x_{d\times1}^T\Psi^T D_\mu \Psi x_{d\times1}=0 \Rightarrow x_{d\times1}=0\}\\
\geq & \mathbb{P}\{\Psi x_{d\times1} =0\Rightarrow x_{d\times1}=0\}=1.
\end{aligned}
\end{equation*}
\end{small}
Therefore, $\Psi^TD_\mu\Psi$ is invertible (a.e.).
\end{proof}

\section{Proof of Lemma 2}

\begin{proof}
Using the Pythagorean theorem, for any measurable function $f\in\mathbb{R}^D$, we have
\begin{small}
\begin{equation*}
\lVert f \rVert _\mu = \lVert \Pi_{\mathcal{G}}f \rVert _\mu+\lVert f -\Pi_{\mathcal{G}}f \rVert _\mu \geq \lVert \Pi_{\mathcal{G}}f\rVert _\mu.
\end{equation*}
\end{small}
Hence, the operator $\Pi_{\mathcal{G}}$ is not an expansion w.r.t. $\mu$-weighted quadratic norm.

Furthermore, from \cite{tsitsiklis1997analysis}, we know that the multi-step Bellman operator
$T^\lambda$ is a $\frac{\gamma(1-\lambda)}{1-\gamma\lambda}$-contraction, i.e.,
\begin{small}
\begin{equation*}
\lVert T^\lambda f_1-T^\lambda f_2 \lVert_\mu \leq \frac{\gamma(1-\lambda)}{1-\gamma\lambda}\lVert  f_1- f_2 \lVert_\mu, \forall f_1, f_2 \in \mathbb{R}^D.
\end{equation*}
\end{small}

Therefore, $\Pi_{\mathcal{G}}T^\lambda$ is also a $\frac{\gamma(1-\lambda)}{1-\gamma\lambda}$-contraction.
\end{proof}

\section{Proof of Theorem 3}

\begin{proof}
For simplicity, denote $\epsilon_A=\hat{A}-A$, and let $\rho(A)$ be the spectral radius of the matrix $A$.

Under Assumption 1, we know that $A$ is invertible by Lemma 1. Consequently, $\hat{A}$ is invertible if and only if $\hat{A}A^{-1}=I+\epsilon_AA^{-1}$ is invertible. According to the relationship between the spectral radius of one matrix and its norm, we can obtain that if $\rho(\epsilon_AA^{-1}) <1$, then it implies that  $\hat{A}A^{-1}=I+\epsilon_AA^{-1}$ is invertible.

From the definition and properties of the matrix norm, we have
\begin{small}
\begin{equation}
\rho(\epsilon_AA^{-1})\leq \lVert \epsilon_AA^{-1} \rVert_2 \leq \lVert\epsilon_A\rVert_2\lVert A^{-1}\rVert_2.
\end{equation}
\end{small}
Therefore, in order to derive the sufficient conditions that $\hat{A}$ is invertible, we just only to need to find the sufficient conditions such that $\lVert\epsilon_A\rVert_2\lVert A^{-1}\rVert_2<1$.
In the following, we would bound $\lVert A^{-1}\rVert_2$ and $\lVert\epsilon_A\rVert_2$ respectively.

\noindent\emph{\textbf{Step 1:} Bound $\lVert A^{-1}\rVert_2$.}

By simple derivations, 
we have
\begin{small}
\begin{equation}
\label{eqn:psiA_inv_decom}
\Psi A^{-1}=(I-\Pi_{\mathcal{G}}M)^{-1}\Psi G^{-1},
\end{equation}
\end{small}
where $M=(1-\lambda)\gamma P(I-\lambda\gamma P)^{-1}$ and $G=\Psi^TD_\mu\Psi$. Furthermore, $(I-\Pi_{\mathcal{G}}M)^{-1}$ is well defined, since
$\lVert \Pi_{\mathcal{G}}M\rVert_\mu
\leq\lVert M\rVert_\mu \leq \frac{(1-\lambda)\gamma}{1-\lambda\gamma}<1$ according to the contraction property of $\Pi_{\mathcal{G}}$ (Lemma 2) and $\lVert P\rVert_\mu =1$ \cite{tsitsiklis1997analysis}.
Besides, by the triangle inequality of the matrix norm, we have
\begin{small}
\begin{equation*}
\label{eqn:Iminusnorm}
\begin{aligned}
&\lVert (I-\Pi_{\mathcal{G}}M)^{-1}\rVert_\mu\leq \lVert \sum_{i=0}^\infty (\Pi_{\mathcal{G}}M)^i\rVert_\mu
\leq \sum_{i=0}^{\infty}\lVert \Pi_{\mathcal{G}}M\rVert_\mu^i
\leq\frac{1-\lambda\gamma}{1-\gamma}.
\end{aligned}
\end{equation*}
\end{small}
Given the random projection $H$, i.e., given $G$, on one hand, for any $g\in\mathbb{R}^d$, we have
\begin{small}
\begin{equation}
\label{eqn:inverse_A_bound_1}
\begin{aligned}
\lVert \Psi A^{-1}g\rVert_\mu\leq \lVert (I-\Pi_{\mathcal{G}}M)^{-1}\rVert_\mu \lVert \Psi G^{-1}g\rVert_\mu
\leq \frac{(1-\lambda\gamma)}{(1-\gamma)\sqrt{\nu_G}}\lVert g\rVert_2.
\end{aligned}
\end{equation}
\end{small}
On the other hand, for any $g\in\mathbb{R}^d$, we have
\begin{small}
\begin{equation}
\label{eqn:inverse_A_bound_2}
\begin{aligned}
\lVert \Psi A^{-1}g\rVert_\mu= \sqrt{(A^{-1}g)^TG(A^{-1}g)}
\geq \sqrt{\nu_G}\lVert A^{-1}g\rVert_2.
\end{aligned}
\end{equation}
\end{small}
Combining Equations (\ref{eqn:inverse_A_bound_1})-(\ref{eqn:inverse_A_bound_2}), by the definition of the operator norm of one matrix, we obtain that
\begin{small}
\begin{equation*}
\lVert A^{-1}\rVert_2 =\sup_{g\neq 0}\frac{\lVert A^{-1}g\rVert_2}{\lVert g\rVert_2}\leq \frac{1-\lambda\gamma}{(1-\gamma)\nu_G}.
\end{equation*}
\end{small}
According to Fact \ref{lem:eigenvalue_g}, with probability at least $1-\delta/2$, we have
\begin{small}
\begin{equation}
\lVert A^{-1}\rVert_2  \leq \frac{1-\lambda\gamma}{(1-\gamma)\nu_0(
\delta/2)},
\end{equation}
\end{small}
\begin{small}
\begin{equation}
\label{eqn:1}
\text{where} \ \nu_0(\delta)=(D/d)\big(1-\sqrt{d/D}-\sqrt{2\log(2/\delta)/D}\big)^2\nu_F.
\end{equation}
\end{small}
\noindent\emph{\textbf{Step 2:} Bound $\lVert\epsilon_A\rVert_2$.}

We first bound $\lVert\mathbb{E}[\epsilon_A]\rVert_2$, and then we can leverage the concentration inequality to derive the upper bound of $\lVert\epsilon_A\rVert_2$.

By the expressions of $A $ and $\hat{A}$, we have
\begin{small}
\begin{equation*}
\begin{aligned}
&\lVert\mathbb{E}[\epsilon_A]\rVert_2\\
=& \lVert\mathbb{E}[\frac{1}{n-1}\sum_{i=1}^{n-1}\sum_{k=-\infty}^{0}(\lambda\gamma)^{i-k}\psi(X_k)(\psi(X_i)-\gamma\psi(X_{i+1}))^T]\rVert_2\\
\leq &\frac{1}{n-1}\frac{2DL^2}{(1-\lambda\gamma)^2}:=\epsilon_1(n).
\end{aligned}
\end{equation*}
\end{small}
For any $\delta' \in (0,1)$, set $\epsilon^2=\frac{8}{d}\log\frac{n}{\delta'}$. So for $d\geq 15\log\frac{n}{\delta'}$, we have $\epsilon < \frac{3}{4}$, and consequently we have
$\epsilon^2/4-\epsilon^3/6 \geq \epsilon^2/8, \ \text{and} \ \  d\geq \frac{1}{\epsilon^2/4-\epsilon^3/6}\log\frac{n}{\delta'}.$
According to Johnson-Lindenstrauss Lemma (Fact \ref{lem:JJL}), with probability at least $1-\delta'$, for all $i\in [1,n]$, we have
\begin{small}
\begin{equation*}
\label{eqn:norm_preservation}
\begin{aligned}
\lVert \psi(X_i)\rVert_2^2=&\lVert H\phi(X_i)\rVert_2^2\leq(1+\epsilon)\lVert \phi(X_i)\rVert_2^2\\
\leq& (1+\epsilon)DL^2
\leq(1+\sqrt{\frac{8}{d}\log\frac{n}{\delta'}})DL^2.
\end{aligned}
\end{equation*}
\end{small}
Define
\begin{small}
\begin{equation}
\label{eqn:2}
\begin{aligned}
\epsilon_1(n,\delta_n)=&\frac{2DLL'}{(1-\lambda\gamma)\sqrt{n-1}}\sqrt{(m_n^{\lambda}+1)J(n-1,\delta_n)}\\
+&\frac{2DLL'}{(n-1)(1-\lambda\gamma)}m_n^\lambda +\epsilon_1(n),
\end{aligned}
\end{equation}
\end{small}
where
$L'=2(1+\sqrt{\frac{8}{d}\log\frac{n}{\delta'}})L,
\ J(n,\delta)= 32\Gamma(n,\delta)\max\{\frac{\Gamma(n,\delta)}{\beta_1},1\}^{\frac{1}{\kappa}},  \ \Gamma(n,\delta)=\log(\frac{2}{\delta})+\log(\max\{4e^2,n\beta_0\}).$

From Fact 5, we know that on the event
\begin{small}
\begin{equation*}
\mathcal{E}_1=\{\lVert \psi(X_i)\rVert_2^2 \leq (1+\sqrt{\frac{8}{d}\log\frac{n}{\delta'}})DL^2, i \in [1,n]\},
\end{equation*}
\end{small}
with probability at least $1-\delta_n$, $\lVert \epsilon_A\rVert_2 \leq \epsilon_1(n,\delta_n)$ holds.

Set $\mathcal{E}_2:=\cup_{n=1}\{\lVert \epsilon_A\rVert_2 \geq \epsilon_1(n,\delta_n)\}$. By the law of total probability, we deduce that
\begin{small}
\begin{equation*}
\begin{aligned}
\mathbb{P}\{\mathcal{E}_2\}
=&\mathbb{P}\{\mathcal{E}_1, \mathcal{E}_2\}
+\mathbb{P}\{\mathcal{E}_1^c,\mathcal{E}_2\}
\leq \mathbb{P}\{\mathcal{E}_2 \vert \mathcal{E}_1\}
+\mathbb{P}\{\mathcal{E}_1^c\}\\
\leq &\sum_{n=1}^{\infty}\mathbb{P}\{\lVert \epsilon_A\rVert_2 \geq \epsilon_1(n,\delta_n)\}\vert \mathcal{E}_1\}
+\mathbb{P}\{\mathcal{E}_1^c\}
\leq (\sum_{n=1}^{\infty}\delta_n)+\delta',
\end{aligned}
\end{equation*}
\end{small}
where $\mathcal{E}_1^c$ is the complement of $\mathcal{E}_1$.

Take $\delta'=\delta/4, \delta_n=\delta/(8n^2)$, we have
\begin{small}
\begin{equation}
\mathbb{P}\{\mathcal{E}_2\}\leq (\sum\nolimits_{n=1}^{\infty}\delta/8n^2)+\delta/4 \leq 8\delta\cdot\pi^2/6+\delta/4\leq \delta/2.
\end{equation}
\end{small}
Therefore, with probability at least $1-\delta/2$, we have
\begin{small}
\begin{equation}
\lVert\epsilon_A\rVert_2 \leq \epsilon_1(n,\frac{\delta}{8n^2}).
\end{equation}
\end{small}
\noindent\emph{\textbf{Step 3:} Derive the sufficient conditions such that $\rho(\epsilon_AA^{-1})<1$.}

Combining the above derivations of step 1 and step 2, we get with probability at least $1-\delta$, the following inequality holds
\begin{small}
\begin{equation}
\label{eqn:3}
\begin{aligned}
\rho(\epsilon_AA^{-1})\leq &\lVert\epsilon_A\rVert_2\lVert A^{-1}\rVert_2
\leq \frac{1-\lambda\gamma}{(1-\gamma)\nu_0(
\delta/2)}\epsilon_1(n,\frac{\delta}{8n^2})<1.
\end{aligned}
\end{equation}
\end{small}
Now we substitute the expressions of $\nu_0(
\delta/2)$ (Equation (\ref{eqn:1})) and $\epsilon_1(n,\frac{\delta}{8n^2})$ (Equation (\ref{eqn:2})) into Equation (\ref{eqn:3}), then we complete the proof.
\end{proof}

\section{Proof of Theorem 4}
Our proof consists of three main steps: First, fixed the low-dimensional space $\mathcal{G}$ which is generated through random projections from the original high-dimensional feature space $\mathcal{F}$, we bound this estimation error based on the results of LSTD($\lambda$) \cite{tagorti2015rate}. Consequently, the bound we obtain depends on the norm of the feature vector of $\mathcal{G}$ and the smallest eigenvalue of the randomly projected Gram matrix $\Psi^TD_\mu\Psi$ which both need to be determined. Then, by utilizing Johnson-Lindenstrauss Lemma (Fact \ref{lem:JJL}), the inner-product preservation property of random projections (Fact \ref{lem:random_linear_norm_preserve}), the relationship between the smallest eigenvalues of the Gram matrices in spaces $\mathcal{F}$ and $\mathcal{G}$ (Fact \ref{lem:eigenvalue_g}) and the concentration inequality for $\beta$-mixing processes,, we bound these two items respectively. Finally, we
summarize all the derivations and get the result.
\begin{proof}
\emph{\textbf{Step 1:} Given $\mathcal{G}$, upper bound the estimation error.}

For any fixed random projected subspace $\mathcal{G}$, according to Theorem 1 in \cite{tagorti2015rate}, for any $\delta_1 >0$, with probability at least $1-\delta_1$ (w.r.t. the random sample), for all $n\geq n_0(\delta_1),$ we have
\begin{small}
\begin{equation}
\label{eqn:given_G_bound}
\begin{aligned}
&\lVert V_{\text{LSTD}(\lambda)\text{-RP}} -\hat{V}_{\text{LSTD}(\lambda)\text{-RP}} \rVert_\mu \\
\leq &\frac{4V_{\max}\max_{1\leq i\leq n}\lVert \psi(X_i)\rVert_2^2}{\sqrt{n-1}(1-\gamma)\nu_{G}}\sqrt{(m_n^\lambda+1)I(n-1,\delta_1)}+h(n,d,\delta_1),
\end{aligned}
\end{equation}
\end{small}
where $\nu_G$ is the smallest eigenvalue of the Gram matrix $G=\Psi^TD_\mu\Psi$, $h(n,d,\delta_1)=\tilde{O}(\frac{d}{n}\log\frac{1}{\delta_1})$. For simplicity, denote the r.h.s. of Equation (\ref{eqn:given_G_bound}) as $EstErr_G(\delta_1)$.


Therefore, we only need to bound the two items \begin{small}$\max\limits_{1\leq i\leq n}\lVert \psi(X_i)\rVert_2^2$\end{small} and $\nu_G$ of $EstErr_G(\delta_1)$ in Equation (\ref{eqn:given_G_bound}).\\
\noindent\emph{\textbf{Step 2:} Bound the two items $\max\limits_{1\leq i\leq n}\lVert \psi(X_i)\rVert_2^2$ and $\nu_G$.}

For any $\delta_2 \in (0,1)$, set $\epsilon^2=\frac{8}{d}\log\frac{n}{\delta_2}$. So for $d\geq 15\log\frac{n}{\delta_2}$, we have $\epsilon < \frac{3}{4}$, and consequently we have
$\epsilon^2/4-\epsilon^3/6 \geq \epsilon^2/8$, and $d\geq \frac{1}{\epsilon^2/4-\epsilon^3/6}\log\frac{n}{\delta_2}.$
According to Johnson-Lindenstrauss Lemma (Fact \ref{lem:JJL}), with probability at least $1-\delta_2$, for all $i\in [1,n]$, we have
\begin{small}
\begin{equation*}
\label{eqn:norm_preservation_copy}
\begin{aligned}
\lVert \psi(X_i)\rVert_2^2
\leq(1+\epsilon)\lVert \phi(X_i)\rVert_2^2
\leq(1+\sqrt{(8/d)\log(n/\delta_2)})DL^2.
\end{aligned}
\end{equation*}
\end{small}
Therefore, with probability at least $1-\delta_2$,
\begin{small}
\begin{equation}
\label{eqn:norm_preservation_1}
\max_{1\leq i \leq n}\lVert \psi(X_i)\rVert_2^2 \leq (1+\sqrt{(8/d)\log(n/\delta_2)})DL^2.
\end{equation}
\end{small}
What's more, for any $\delta_3 \in (0,1)$, from Fact \ref{lem:eigenvalue_g}, when $D>d+2\sqrt{2d\log(2/\delta_3)}+2\log(2/\delta_3)$, with probability at least $1-\delta_3$, we have
\begin{small}
\begin{equation}
\label{eqn:egivalue}
\nu_G \geq \nu_0(\delta_3):=(D/d)\nu_F\big(1-\sqrt{d/D}-\sqrt{2\log(2/\delta_3)/D}\big)^2>0.
\end{equation}
\end{small}
\noindent\emph{\textbf{Step 3:} Bound the estimation error \begin{small}$\lVert V_{\text{LSTD}(\lambda)\text{-RP}} -\hat{V}_{\text{LSTD}(\lambda)\text{-RP}} \rVert_\mu$\end{small}.}

In this step, we bridge the estimation error between the fixed random projection space $\mathcal{G}$ and any arbitrary random projection space by the conditional expectation properties. Combining Equations (\ref{eqn:given_G_bound})-(\ref{eqn:egivalue}), unconditioning, we have
\begin{small}
\begin{equation}
\begin{aligned}
&\mathbb{P}\big\{\lVert V_{\text{LSTD}(\lambda)\text{-RP}} -\hat{V}_{\text{LSTD}(\lambda)\text{-RP}} \rVert_\mu
\leq EstErr_G(\delta_1) \big\}\\
\geq &\mathbb{P}\big\{\lVert V_{\text{LSTD}(\lambda)\text{-RP}} -\hat{V}_{\text{LSTD}(\lambda)\text{-RP}} \rVert_\mu \leq EstErr_G(\delta_1),
(\max_{1\leq i\leq n}\lVert \psi(X_i)\rVert_2^2)/\nu_{G}\leq DL^2\xi(n,d,\delta_2)/\nu_0(\delta_3)\big\}\\
=&\mathbb{E}\big[\mathbb{I}\{(\max_{1\leq i\leq n}\lVert \psi(X_i)\rVert_2^2)/\nu_{G}\leq DL^2\xi(n,d,\delta_2)/\nu_0(\delta_3)\}\times\\
&\mathbb{P}\{\lVert V_{\text{LSTD}(\lambda)\text{-RP}} -\hat{V}_{\text{LSTD}(\lambda)\text{-RP}} \rVert_\mu
 \leq EstErr_G(\delta_1) \big\vert H\}\big]\\
\geq &(1-\delta_1)\mathbb{P}\{(\max_{1\leq i\leq n}\lVert \psi(X_i)\rVert_2^2)/\nu_{G}\leq DL^2\xi(n,d,\delta_2)/\nu_0(\delta_3)\}\\
\geq & (1-\delta_1)(1-\delta_2-\delta_3)\geq 1-\delta_1-\delta_2-\delta_3.
\end{aligned}
\end{equation}
\end{small}
Setting $\delta_1=\delta_2=\delta_3=\delta/3$, then we complete the proof.
\end{proof}

\section{Proof of Theorem 5}
For the proof brevity of Theorem 5, we first present the following Lemma \ref{Prop:chernoff_hoeffing_mixing} which is important during the proof process of Theorem 5. To prove Lemma \ref{Prop:chernoff_hoeffing_mixing}, we first make full use of \emph{independent block technique} \cite{yu1994rates} in order to transform the original problem based on dependent samples to that based on independent blocks. Then, we apply the symmetrization technique and Hoeffding inequality to obtain the desired bound.
\begin{lemma}
\label{Prop:chernoff_hoeffing_mixing}
Let $\{X_t\}_{t=1}^n$ be samples drawn from a stationary exponential $\beta$-mixing process with coefficients satisfy $\beta(m)\leq \beta_0\exp(-\beta_1m^\kappa)$, $\beta_0, \beta_1, \kappa>0$. Let $h \in \mathcal{B}(\mathcal{X},M_h)$). Then for any $\delta \in (0,1)$, with probability at least $1-\delta$, we have\\
\begin{small}
\begin{equation*}
\lvert\frac{1}{n}\sum\nolimits_{t=1}^nh(X_t)-\mathbb{E}h(X_t)\rvert\leq \frac{2M_h}{\sqrt{n}}\sqrt{\Upsilon(n,\delta)},
\end{equation*}
\end{small}
where $\Upsilon(n,\delta)=[\log\frac{1}{\delta}+\log(4+n\beta_0)]\big[\frac{\log\frac{1}{\delta}+\log(4+n\beta_0)}{\beta_1}\big]^{\frac{1}{\kappa}}$.
\end{lemma}

\begin{proof}
Denote $u_n=n/2a_n$\footnote{Without loss of generality, here we assume $n=2u_na_n$. If $n$ is an odd number, this lemma still holds.}, where $u_n\in \mathbb{N}^+, a_n\in\mathbb{N}^+$.
Divide $\{X_t\}_{t=1}^n$ into $2u_n$ blocks, each of which consists of $a_n$ consecutive samples. For $1 \leq j \leq u_n$, we define $H_j = \{t: 2(j-1)a_n + 1 \leq t \leq (2j-1)a_n\}$, and $T_j = \{t: (2j-1)a_n + 1 \leq t \leq (2j)a_n\}$. We introduce i.i.d. blocks $\{\tilde{X}_{t}: t \in H_j\}$ and each block has the same distribution with $\{X_{t}: t \in H_1\}$. Therefore, for any $\epsilon>0$
\begin{small}
\begin{equation}
\label{eqn:hoeffeding_mixing}
\begin{aligned}
&\mathbb{P}\{\lvert\frac{1}{n}\sum_{t=1}^nh(X_t)-\mathbb{E}h(X_t)\rvert\geq\epsilon\}\\
\leq &\mathbb{P}\{\lvert\sum_{j=1}^{u_n}\sum_{t\in H_j}h(X_t)-\mathbb{E}h(X_t)\rvert
+\lvert\sum_{j=1}^{u_n}\sum_{t\in T_j}h(X_t)-\mathbb{E}h(X_t)\rvert\geq n\epsilon/2\}\\
\leq &\mathbb{P}\{\lvert\sum_{j=1}^{u_n}\sum_{t\in H_j}h(X_t)-\mathbb{E}h(X_t)\rvert\geq n\epsilon/4\}
+\mathbb{P}\{\lvert\sum_{j=1}^{u_n}\sum_{t\in T_j}h(X_t)-\mathbb{E}h(X_t)\rvert\geq n\epsilon/4\}\\
=&2\mathbb{P}\{\lvert\sum_{j=1}^{u_n}\sum_{t\in H_j}h(X_t)-\mathbb{E}h(X_t)\rvert\geq n\epsilon/4\}\\
\leq &2\mathbb{P}\{\lvert\sum_{j=1}^{u_n}\sum_{t\in H_j}h(\tilde{X}_t)-\mathbb{E}h(\tilde{X}_t)\rvert\geq n\epsilon/4\}+2u_n\beta_{a_n}\\
\leq &4\exp(-\frac{u_n\epsilon^2}{2M_h^2})+2u_n\beta_{a_n}\\
\leq&4\exp(-\frac{n\epsilon^2}{4a_nM_h^2})+\frac{n}{a_n}\beta_0\exp(-\beta_1a_n^\kappa)\\
\leq &4\exp(-\frac{n\epsilon^2}{4a_nM_h^2})+n\beta_0\exp(-\beta_1a_n^\kappa).
\end{aligned}
\end{equation}
\end{small}
Take $a_n=\lceil (\frac{n\epsilon^2}{4\beta_1M_h^2})^{\frac{1}{\kappa+1}}\rceil$, Then Equation (\ref{eqn:hoeffeding_mixing}) becomes:
\begin{small}
\begin{equation}
\begin{aligned}
\mathbb{P}\{\lvert\frac{1}{n}\sum_{t=1}^nh(X_t)-\mathbb{E}h(X_t)\rvert\geq\epsilon\}
\leq (4+n\beta_0)\exp(-\beta_1(\frac{n\epsilon^2}{4\beta_1M_h^2})^{\frac{\kappa}{\kappa+1}}).
\end{aligned}
\end{equation}
\end{small}
Then if we set $\delta=(4+n\beta_0)\exp(-\beta_1(\frac{n\epsilon^2}{4\beta_1M_h^2})^{\frac{\kappa}{\kappa+1}})$, we can complete our proof.
\end{proof}

The proof of Theorem 5 is organized in three main steps: First, we analyze the approximation error on any fixed random projected space $\mathcal{G}$. Then, we make a bridge of approximation error bound between the fixed random projection space and any arbitrary random projection space by leveraging the inner-product preservation property of random projections (Fact \ref{lem:random_linear_norm_preserve}) and the Chernoff-Hoeffding inequality for stationary $\beta$-mixing sequence (Lemma \ref{Prop:chernoff_hoeffing_mixing}). Finally, we summarize all the derivations and eventually get the theorem.

\noindent\textbf{Proof of Theorem 5}\\
\noindent\emph{\textbf{Step 1:} Given $\mathcal{G}$, upper bound the approximation error.}

Fixed the random projected space $\mathcal{G}$, according to Theorem 1 in \cite{tsitsiklis1997analysis}, we have
\begin{small}
\begin{equation}
\label{eqn:app_err_Given_G}
\lVert V-V_{\text{LSTD}(\lambda)\text{-RP}}\rVert_\mu \leq (1-\lambda\gamma)/(1-\gamma)\lVert V-\Pi_{\mathcal{G}}V\rVert_\mu.
\end{equation}
\end{small}
By using the triangle inequality of the norm, we have
\begin{small}
\begin{equation}
\label{eqn:tri_v}
\lVert V-\Pi_{\mathcal{G}}V\rVert_\mu \leq \lVert V-\Pi_\mathcal{F}V\rVert_\mu + \lVert \Pi_\mathcal{F}V-\Pi_{\mathcal{G}}V\rVert_\mu.
\end{equation}
\end{small}
Since $\Pi_\mathcal{F}V$ is the orthogonal projection of $V$ on the high-dimensional space $\mathcal{F}$, for any $g \in \mathcal{G} \subseteq \mathcal{F}$, using the Pythagorean theorem, we have
\begin{small}
\begin{equation*}
\lVert V-g\rVert_\mu^2 = \lVert V-\Pi_\mathcal{F}V\rVert_\mu^2 + \lVert \Pi_\mathcal{F}V-g\rVert _\mu^2.
\end{equation*}
\end{small}
Therefore,
\begin{small}
\begin{equation*}
\arg\inf_{g\in\mathcal{G}}\lVert V-g\rVert_\mu = \arg\inf_{g\in\mathcal{G}}\lVert \Pi_\mathcal{F}V-g\rVert _\mu.
\end{equation*}
\end{small}
According to the definition of the orthogonal projection, we obtain that
\begin{small}
\begin{equation}
\label{eqn:proj}
\Pi_\mathcal{G}V=\Pi_\mathcal{G}(\Pi_\mathcal{F}V).
\end{equation}
\end{small}
Combine Equations (\ref{eqn:app_err_Given_G})-(\ref{eqn:proj}), we obtain that
\begin{small}
\begin{equation}
\label{eqn:v_decom}
\begin{aligned}
&\lVert V-V_{LSTD(\lambda)-RP}\rVert_\mu\\
\leq & (1-\lambda\gamma)/(1-\gamma)[\lVert V-\Pi_\mathcal{F}V\rVert_\mu + \lVert \Pi_\mathcal{F}V-\Pi_{\mathcal{G}}(\Pi_\mathcal{F}V)\rVert_\mu].
\end{aligned}
\end{equation}
\end{small}
Therefore, next we only need to bound $\lVert \Pi_\mathcal{F}V-\Pi_{\mathcal{G}}(\Pi_\mathcal{F}V)\rVert_\mu$ in the r.h.s. of Equation (\ref{eqn:v_decom}).\\
\noindent\emph{\textbf{Step 2:} Bound $\lVert \Pi_\mathcal{F}V-\Pi_{\mathcal{G}}(\Pi_\mathcal{F}V)\rVert_\mu$.}

For ease the reference, let $Z(x)=\lvert \alpha \cdot\phi(x)-H\alpha\cdot H\phi(x)\rvert$.
For any $\delta>0$, take $\epsilon=\sqrt{(8/d)\log(8n/\delta)}$. For $d\geq 15\log(8n/\delta)$, we have $\epsilon\leq 3/4$ and accordingly,
\begin{small}
$\epsilon^2/4-\epsilon^3/6 \geq \epsilon^2/8, \ \text{and} \ \  d\geq \frac{\log(8n/\delta)}{\epsilon^2/4-\epsilon^3/6}.$
\end{small}
According to Fact \ref{lem:random_linear_norm_preserve}, we have with probability at least $1-\delta/2$, for all $1\leq i\leq n$, the following inequality holds,
\begin{small}
\begin{equation*}
Z(X_i)\leq \epsilon\lVert \alpha\rVert_2\lVert \phi(X_i)\rVert_2\leq \epsilon m(f_\alpha).
\end{equation*}
\end{small}
Denote
\begin{small}
$\epsilon_0=(2\epsilon m(f_\alpha)/\sqrt{n})\sqrt{\Upsilon(n,\delta/2)}.$
\end{small}
According to Chernoff-Heoffding's inequality for stationary exponential $\beta$-mixing sequences (see Lemma \ref{Prop:chernoff_hoeffing_mixing}), we have
\begin{small}
\begin{displaymath}
\begin{aligned}
&\mathbb{P}\{\lvert\mathbb{E} Z(X) - (1/n)\sum\nolimits_{i=1}^n Z(X_i)\rvert \geq \epsilon_0\}\\
\leq & \mathbb{P}\{\lvert\mathbb{E} Z(X) - (1/n)\sum\nolimits_{i=1}^n Z(X_i)\rvert \geq \epsilon_0, Z(X_i)\leq \epsilon m(f_\alpha), \forall i\in [1,n]\}\\
+&\mathbb{P}\{Z(X_i)\geq \epsilon m(f_\alpha), \exists i\in [1,n]\}\\
\leq &\mathbb{P}\big\{\mathbb{P}\{\lvert\mathbb{E} Z(X) - (1/n)\sum\nolimits_{i=1}^n Z(X_i)\rvert \geq \epsilon_0 \big\vert H, Z(X_i)\leq \epsilon m(f_\alpha),\forall i\in [1,n]\}\big\}
+\delta/2\\
\leq &\delta.
\end{aligned}
\end{displaymath}
\end{small}
Therefore, with probability (w.r.t. the random sample and random projection) at least $1-\delta$, we have
\begin{small}
\begin{displaymath}
\begin{aligned}
\lVert\alpha \cdot\phi(x)-H\alpha\cdot H\phi(x)\rVert_\mu =& \mathbb{E}[Z(X)]
\leq \frac{1}{n}\sum\nolimits_{i=1}^n Z(X_i)+\epsilon_0 \leq \epsilon m(f_\alpha)+\epsilon_0
\end{aligned}
\end{displaymath}
\end{small}
As a result, with probability (w.r.t. the random sample and random projection) at least $1-\delta$, we have
\begin{small}
\begin{equation}
\label{eqn:bound_F_projection}
\begin{aligned}
&\lVert \Pi_\mathcal{F}V-\Pi_{\mathcal{G}}(\Pi_\mathcal{F}V)\rVert_\mu \\
\leq &\sqrt{(8/d)\log(8n/\delta)} m(\Pi_{\mathcal{F}}V)[1+(2/\sqrt{n})\sqrt{\Upsilon(n,\delta/2)}].
\end{aligned}
\end{equation}
\end{small}
\noindent\emph{\textbf{Step 3:} Bound the approximation error $\lVert V-V_{\text{LSTD}(\lambda)\text{-RP}}\rVert_\mu$.}

In this step, we make a bridge of the approximation error between the fixed random projection space $\mathcal{G}$ and arbitrary random projection space through the conditional expectation properties. Combining Equations (\ref{eqn:v_decom})-(\ref{eqn:bound_F_projection}), we get the theorem.$\hfill\square$


\bibliographystyle{plain}
\bibliography{ijcai18}

\end{document}